\documentclass[sigconf]{acmart}
\usepackage{hyperref}
\usepackage{xcolor}
\usepackage{colortbl}
\usepackage{caption}
\usepackage{float}
\usepackage{multirow}
\usepackage{enumitem}
\usepackage[symbol]{footmisc}
\usepackage{hhline}
\usepackage{bbm}
\usepackage{tikz}

\usepackage{amsmath}
\DeclareMathOperator*{\argmin}{argmin} % no space, limits underneath in displays
\DeclareMathOperator*{\argmax}{argmax} % no space, limits underneath in displays
\AtBeginDocument{%
  \providecommand\BibTeX{{%
    \normalfont B\kern-0.5em{\scshape i\kern-0.25em b}\kern-0.8em\TeX}}}

\setcopyright{acmcopyright}
\copyrightyear{2018}
\acmYear{2018}
\acmDOI{XXXXXXX.XXXXXXX}

\acmConference[Conference acronym 'XX]{Make sure to enter the correct
  conference title from your rights confirmation email}{June 03--05,
  2018}{Woodstock, NY}
\acmPrice{15.00}
\acmISBN{978-1-4503-XXXX-X/18/06}

\usepackage{microtype}
\usepackage{graphicx}
\usepackage{subfigure}
\usepackage{booktabs} % for professional tables

\usepackage{amsmath}
\usepackage{amssymb}
\usepackage{mathtools}
\usepackage{amsthm}

\theoremstyle{plain}
\newtheorem{theorem}{Theorem}[section]
\newtheorem{proposition}[theorem]{Proposition}
\newtheorem{lemma}[theorem]{Lemma}
\newtheorem{corollary}[theorem]{Corollary}
\theoremstyle{definition}
\newtheorem{definition}[theorem]{Definition}

\theoremstyle{remark}

\begin{document}

% \twocolumn[
% \icmltitle{Intersectional Divergence: Measuring Fairness in Regression}
\title{Intersectional Divergence: Measuring Fairness in Regression}
% \begin{icmlauthorlist}
% \icmlauthor{Joe Germino}{nd}
% \icmlauthor{Nuno Moniz}{nd}
% \icmlauthor{Nitesh V. Chawla}{nd}
% \end{icmlauthorlist}

% \author{Joe Germino}
% \email{jgermino@nd.edu}
% \orcid{0009-0007-3308-1715}
% \affiliation{%
%   \institution{Lucy Family Institute for Data \& Society, University of Notre Dame}
%   \city{Notre Dame}
%   \state{IN}

% \author{Nuno Moniz}
% \orcid{0000-0003-4322-1076}
% \email{nuno.moniz@nd.edu}
% \affiliation{%
%   \institution{Lucy Family Institute for Data \& Society, University of Notre Dame}
%   \city{Notre Dame}
%   \state{IN}

% \author{Nitesh V. Chawla}
% \email{nchawla@nd.edu}
% \orcid{0000-0003-3932-5956}
% \affiliation{%
%   \institution{Lucy Family Institute for Data \& Society, University of Notre Dame}
%   \city{Notre Dame}
%   \state{IN}

\author{Joe Germino, Nuno Moniz, Nitesh V. Chawla}
\email{{jgermino, nuno.moniz, nchawla}@nd.edu}
\affiliation{%
  \institution{Lucy Family Institute for Data \& Society, University of Notre Dame}
  \city{Notre Dame}
  \state{IN}
  \country{USA}
}

% \icmlaffiliation{nd}{Lucy Family Institute for Data \& Society, University of Notre Dame, Notre Dame, IN, USA}

% \icmlcorrespondingauthor{Nitesh V. Chawla}{nchawla@nd.edu}

% \icmlkeywords{Fairness, Intersectionality, Imbalanced Data, Regression}
% \vskip 0.3in
% ]



% \printAffiliationsAndNotice{}

\begin{abstract}
Fairness in machine learning research is commonly framed in the context of classification tasks, leaving critical gaps in regression. In this paper, we propose a novel approach to measure intersectional fairness in regression tasks, going beyond the focus on single protected attributes from existing work to consider combinations of all protected attributes. Furthermore, we contend that it is insufficient to measure the average error of groups without regard for imbalanced domain preferences. Accordingly, we propose \textbf{Intersectional Divergence (ID)} as the first fairness measure for regression tasks that 1) describes fair model behavior across multiple protected attributes and 2) differentiates the impact of predictions in target ranges most relevant to users. We extend our proposal demonstrating how \textbf{ID} can be adapted into a loss function, \textbf{IDLoss}, that satisfies convergence guarantees and has piecewise smooth properties that enable practical optimization. Through an extensive experimental evaluation, we demonstrate how \textbf{ID} allows unique insights into model behavior and fairness, and how incorporating \textbf{IDLoss} into optimization can considerably improve single-attribute and intersectional model fairness while maintaining a competitive balance in predictive performance.
\end{abstract}
\maketitle

\section{Introduction}

There are two critical aspects of growing importance in Fair Machine Learning: the recognition of the intersectionality of protected attributes and the impact of imbalanced domains.
While fairness is most commonly measured as the difference in performance between groups across a single protected attribute~\cite{calders2013controlling, berk2017convex,  pmlr-v97-agarwal19d, pmlr-v139-chi21a}, this approach is severely limiting and can hide model biases~\cite{pmlr-v81-buolamwini18a}. Instead, one must consider the simultaneous impact of multiple protected attributes, i.e., \textit{intersectionality}. Furthermore, it is necessary to acknowledge the impact of imbalanced domains. Depending on the context of the task, some values may be more important to predict accurately than others, introducing an additional layer to fairness.

Thus far, most existing fairness work has been focused on classification tasks with negligible attention towards regression~\cite{pessach2022review, chen2022fairness, 10.1145/3616865}. 
While the issues of intersectionality~\cite{gohar2023survey, xu2024intersectional} and imbalance~\cite{iosifidis2020online, roy2022multi} have been addressed in classification~\cite{morina2019auditing, foulds2020intersectional} and ranking tasks~\cite{pastor2024intersectional}, the level of attention to these issues in regression has been negligible in comparison~\cite{zhao2019rank}. Importantly, this has left a critical gap where no work is available to tackle intersectional fairness, while accounting for imbalanced domain preferences in regression.

\begin{figure}[!t]
    \centering
    \includegraphics[width=\linewidth]{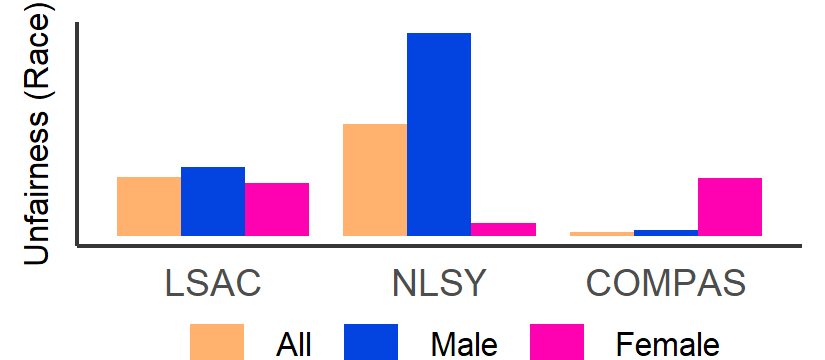}
    \caption{Unfairness measured by the difference in MAE split by race in 3 different datasets. All measures disparate treatment in race across all persons both Male and Female. Disparate treatment based on race can vary widely based on sex (Male and Female). Results from Experimental Evaluation (Section~\ref{sec:EE}). 
    }
    % Measure of unfairness by race in 3 different data sets. The ``All" line illustrates the total unfairness present in the model. The other bars show the unfairness of race decomposed by sex. Unfairness is measured with MAE difference and to the Total Error per data set. Results taken from Experimental Evaluation (Section~\ref{sec:EE}).}
    \label{fig:intro1}
\end{figure}

We demonstrate the intersectionality problem in Figure~\ref{fig:intro1}. 
Unfairness by race varies significantly depending on the sex attribute. For example, in the COMPAS dataset, although overall unfairness by race is near zero, females are subject to disparate treatment based on their race. This shows why using a single feature to evaluate the model bias sources hinders more actionable explanations. As for the consequences of domain imbalance, in Figure~\ref{fig:intro2}, we present a synthetic scenario where a financial firm is using a model to assign clients a risk score. Errors that misidentify high-risk clients as low-risk are more costly than the opposite. Because the Mean Absolute Error (MAE) is equal for both privileged and unprivileged groups, this scenario would appear fair by traditional fairness measures. However, error distributions differ significantly -- the MAE for the privileged group is lower than the MAE for the unprivileged group when predicting high-risk scores. Such discrepancies may hold substantial real-world implications.

\begin{figure}
    \centering
    \includegraphics[width=\linewidth]{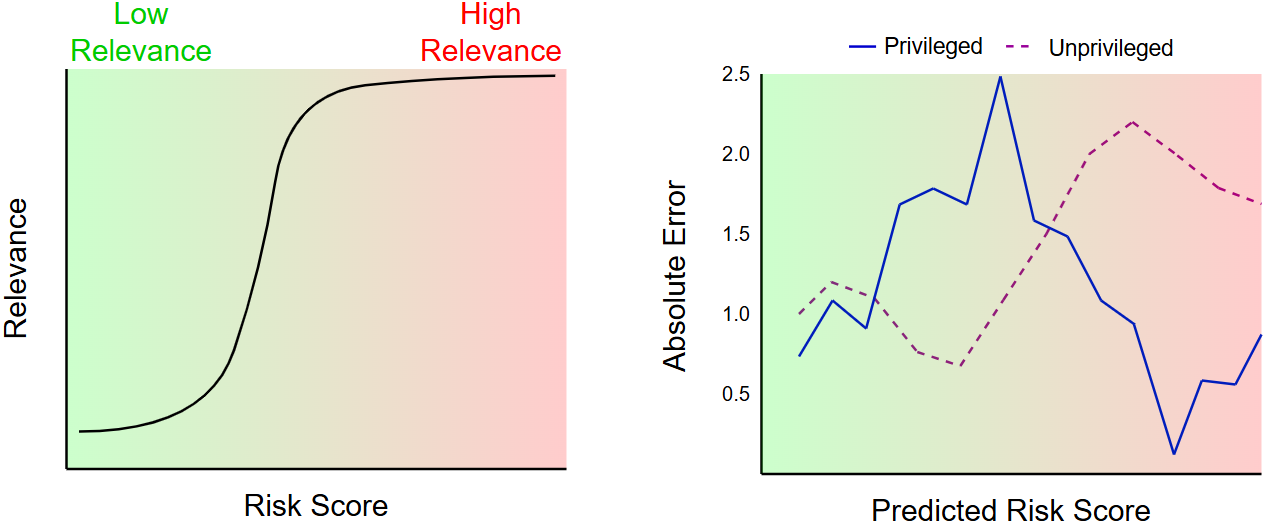}
    \caption{
    % An artificial example illustrating domain imbalance. In both images, the x-axis cor On the left, the y-axis corresponds to the importance of  
    A hypothetical scenario where higher values are more important to predict accurately (left) and the total error for both groups is identical (right) despite the unprivileged group having significantly higher errors in the higher relevance values.}    

    \label{fig:intro2}
\end{figure}

\paragraph{Contributions.} In this paper, we illustrate the urgency in considering intersectionality and domain imbalance in fair regression measures. We propose a new measure, \textbf{Intersectional Divergence (ID)}, providing a more accurate representation of a model's biases and a deeper understanding of unfair behavior w.r.t. existing methods. Finally, we demonstrate how ID can be adapted into a loss function, IDLoss, and provide theoretical analysis establishing its convergence properties and optimization guarantees despite its non-convex nature. Our analysis shows that IDLoss satisfies the \L{}ojasiewicz inequality~\cite{karimi2020linearconvergencegradientproximalgradient}, ensuring convergence to stationary points, and has piecewise Lipschitz continuous gradients that enable practical optimization.

\section{Related Work}

Existing fairness measures can be separated into three categories: group, individual~\cite{cynthia2012fairness, pmlr-v235-bechavod24a, pmlr-v235-munagala24a}, and counterfactual fairness~\cite{kusner2017counterfactual}. The most common of these is \textit{group fairness}, which attempts to ensure that privileged and unprivileged groups are treated equally.
For example, Statistical Parity measures the probability difference between the privileged and unprivileged groups in the positive class~\cite{cynthia2012fairness}, and Equalized Odds measures the difference in the fraction of true and false positives between groups~\cite{hardt2016equality}. ABROCA measures the difference between groups' Receiver Operating Characteristics (ROC) curves~\cite{gardner2019evaluating}. 

Although there is a focus on measuring and optimizing around single protected attributes, there are strong arguments about the need for a different approach. Critically,~\citet{crenshaw2013demarginalizing} discusses the theory of intersectionality and argues that unique combinations of protected attributes interact in their own ways, which can lead to bias not captured on an individual level. Alternatively, multiple discrimination explores the additive effects of discrimination across multiple protected attributes~\cite{roy2023multi}, and~\citet{alvarez2025counterfactual} demonstrate that multiple discrimination fails to account for the intersectional biases.

\citet{pmlr-v81-buolamwini18a} show how race and gender intersectionality affects the errors of a facial recognition system, and \citet{colakovic2023fairboost} explores boosting with multiple sensitive attributes. Fair classification algorithms apply these measures to an optimization problem in various ways. \citet{zafar2017fairness} uses an in-processing approach that applies fairness constraints to a classifier while maximizing performance, and \citet{agarwal2018reductions} uses an adversarial learning approach to exploit failures in fairness. Alternatively, \citet{kamiran2012data} develops a pre-processing method to remove unfairness through data relabeling, and \citet{10.1145/3468264.3468537} uses data perturbation and sampling. \citet{zhang2024fair} proposes a general framework for calibrating models based on fairness risks across multiple sensitive attributes simultaneously.

\subsection{Fairness in Regression}
In regression, group fairness is the common approach and the one we use in \textbf{ID}. Measures such as Statistical Parity compare the difference in the predicted CDF between groups using the Kolmogorov-Smirnov statistic~\cite{pmlr-v97-agarwal19d}. Mean Difference~\cite{calders2013controlling} and Bounded Group Loss~\cite{pmlr-v97-agarwal19d} calculate the difference of average predictions and the error difference between groups, respectively. \citet{berk2017convex} proposes approaches to measuring group fairness and individual fairness, as well as a hybrid approach that measures both simultaneously. 

On fair regression algorithms, \citet{fitzsimons2019general} proposes a general framework in which fairness constraints are included in kernel regression, and \citet{pmlr-v80-komiyama18a} demonstrates how nonconvex optimization can be utilized to include fairness constraints while minimizing loss. \citet{mohamed2022normalise} uses a pre-processing algorithm to remove unfairness by normalizing the target variable before fitting a model, and \citet{perez2017fair} proposes a Fair Dimensionality Reduction framework to remove unfairness through feature embedding. Finally, \citet{chzhen2020fair} proposes a post-processing algorithm using Wasserstein barycenters to learn an optimal fair predictor.

Regarding intersectionality, \citet{herlihy2024structured} introduces a regression approach using confidence intervals to measure intersectional groups' performance and demonstrate that strong performance can be achieved even with small samples.
Also, several approaches have evaluated fairness on non-binary protected attributes, posing similar challenges. In classification tasks, \citet{duong2023towards} proposes using the sum of absolute differences or the maximal absolute difference of all potential values. Alternatively, \citet{celis2021fair} uses multiplicative fairness constraints to measure the performance ratio for the best and worst groups. However, none of these approaches consider imbalanced domains.

\subsection{Learning with Imbalanced Domains}
Pre-processing algorithms have previously been used to address imbalance in fairness classification tasks. \citet{sonoda2023fair} proposes FairSMOTE leveraging over-sampling techniques on heterogeneous clusters, and \citet{9951398} proposes FairMask. This extrapolation method represents protected attributes through models trained on the other independent variables. Thus far, attempts to correct for imbalance in fair learning have been limited to classification.

Fairness measures in regression tasks do not consider the impact of imbalanced data. Nonetheless, previous work on solving imbalanced regression tasks exists~\cite{Torgo2006ImbReg}.  SMOTEBoost~\cite{8631400} demonstrates how a boosting technique can improve the prediction of extreme values. SERA is an error measure that explicitly considers the importance of accurately predicting non-uniform domain preferences~\cite{ribeiro2020imbalanced}, and it has previously been used as an optimization function to directly consider this imbalance~\cite{silva2022model}.
\paragraph{Novelty.} To the best of our knowledge, \textbf{ID} is the first fairness measure for regression tasks that considers the intersectionality of protected attributes and accounts for domain imbalance. \textbf{IDLoss} can be used to optimize models while considering intersectionality and imbalance and maintaining strong predictive performance.

\section{Background}
Squared-Error Relevance Area (SERA) measures the predictive performance of a model while considering domain imbalance~\cite{ribeiro2020imbalanced}. SERA uses a continuous, domain-dependent relevance function $\phi(Y): \mathcal{Y} \rightarrow [0,1]$ to express the application-specific bias concerning the target variable $\mathcal{Y}$. The relevance function is defined by a domain expert indicating which target values are considered low or high-relevance. In lieu of such domain information, the function can be interpolated from boxplot-based statistics where extreme values are considered high-relevance and the distribution median the lowest point of relevance. 

\definition (SERA). Let $A, X, Y$ represent protected features, remaining features, and the output of interest, respectively. Given a dataset $\mathcal{D} = \{\langle X_i, A_i, y_i\rangle\}_{i=1}^N$ and relevance function $\phi(Y): \mathcal{Y} \rightarrow [0,1]$, $\mathcal{D}^t \subseteq \mathcal{D}$ is the subset of cases with target value relevance above or equal to cutoff $t$, i.e., $\mathcal{D}^t = \{\langle X_i, A_i, y_i\rangle \in \mathcal{D}\mid \phi(y_i) \geq t\}$.
The Squared Error-Relevance concerning a cutoff $t$ $(SER^t)$ is the sum of the squared error for all samples in $\mathcal{D}^t$:
\begin{equation}
    SER^t = \sum_{i\in \mathcal{D}^t} {(\hat{y}_i - y_i)^2}
\end{equation}
where $\hat{y}_i$ and $y_i$ are the predicted and true values for case $i$.

Given this, SERA is the area under the curve represented by $SER^t$ for all possible relevance cutoffs $t \in [0,1]$:
\begin{equation}
    SERA = \int_0^1 SER^t dt = \int_0^1\sum_{i\in \mathcal{D}^t} {(\hat{y}_i - y_i)^2}dt
\end{equation}
Intuitively, integrating over a relevance cutoff $t$, SERA considers the error for all samples with greater weight given to high-relevance cases. \citet{silva2022model} proved SERA is twice-differentiable and demonstrated how to implement it as a loss function. 

\section{Intersectional Divergence}

In this paper, we propose \textbf{Intersectional Divergence (ID)} as a measure of fairness in regression tasks. \textbf{ID} considers the difference in error curves weighted by relevance for each subgroup of protected attributes, measuring the area of maximum divergence in error between all subgroups corresponding to the combinations of binary-protected attributes.

\definition (Intersectional Divergence). Given protected attributes $A$, let $\mathcal{A}$ represent all possible combinations of values within $A$ and $\alpha$ be a given combination. $\mathcal{D}_\alpha \subseteq \mathcal{D}$ is defined as the cases for which the protected attribute combination of a sample is equal to $\alpha$, i.e. $\mathcal{D}_\alpha = \{\langle X_i, A_i, y_i\rangle \in \mathcal{D}\mid A_i=\alpha\}$, and $\mathcal{D}^t_\alpha = \mathcal{D}^t \cap \mathcal{D}_\alpha$. Then, $SER^t_\alpha$ represents the Squared Error-Relevance for a single combination of protected attributes above or equal to a relevance value $t$,
\begin{equation}
    SER^t_\alpha = \sum_{i\in \mathcal{D}_\alpha^t} {(\hat{y}_i - y_i)^2}  
\end{equation}
We define $\alpha$ with the maximum and minimum $SER$ values at each $t$ respectively as:
\begin{equation}
    \alpha_{max} = \argmax_{\alpha \in \mathcal{A}}(\frac{SER^t_\alpha}{|\mathcal{D}_\alpha^t|})
\end{equation}
\begin{equation}
    \alpha_{min} = \argmin_{\alpha \in \mathcal{A}}(\frac{SER^t_\alpha}{|\mathcal{D}_\alpha|})
\end{equation}
\textbf{ID} is the area between the curves for the maximum $SER^t_\alpha$ and minimum $SER^t_\alpha$ at every $t$ normalized by the subset size with protected attributes $\alpha$ and relevance $t$.
\begin{equation}
        ID = \int_0^1 \frac{SER_{\alpha_{max}}^t}{|\mathcal{D}_{\alpha_{max}}^t|}  -
        \frac{SER_{\alpha_{min}}^t}{|\mathcal{D}_{\alpha_{min}}^t|}  dt
\end{equation}
\begin{equation}
    \begin{split}
         = \int_0^1 \frac{\sum_{i\in \mathcal{D}_{\alpha_{max}}^t} {(\hat{y}_i - y_i)^2}}{|\mathcal{D}_{\alpha_{max}}^t|}  -
         \frac{\sum_{i\in \mathcal{D}_{\alpha_{min}}^t} {(\hat{y}_i - y_i)^2}}{|\mathcal{D}_{\alpha_{min}}^t|}  dt    
    \end{split}    
\end{equation}
\paragraph{Intuition.} \textbf{ID} measures the area difference between the maximum and minimum SER curves, thereby measuring the divergence between the best- and worst-predicted group at every relevance threshold. \textbf{ID} ensures that no group has a significantly higher error than another while adjusting for domain relevance.
The ideal value of \textbf{ID} is 0 -- identical error for all the protected groups.
% \begin{figure*}
%     \centering
%     \includegraphics[width=.85\linewidth]{figures/flowchart_horiz.png}
%     \caption{Overview of IDBoost framework containing two ensembles optimized for predictive performance (SERA) and  fairness (IDLoss).}
%     \label{fig:flow}
% \end{figure*}
\subsection{A Loss Function for ID}
Directly optimizing for \textbf{ID} may present problems in predictive performance. A model may learn to increase the error of the best-performing group towards the worst-performing group. This could result in less divergence at the expense of an increase in the total error, an undesired outcome. Instead, we demonstrate how \textbf{ID} can be transformed into a twice-differentiable optimization loss function which will decrease divergence without degrading predictive performance.

\definition (IDLoss). We propose \textbf{IDLoss} as the sum of $SER^t$ for all $\alpha$ excluding $\alpha_{min}$, lowering the error of each protected group towards the group with the smallest error, decreasing both divergence and total error.
\begin{equation}
    IDLoss =  \int_0^1\sum_{\alpha \in \mathcal{A} \setminus \alpha_{min}} \frac{SER_{\alpha}^t}{|\mathcal{D}_{\alpha}^t|}  dt    
\end{equation}
\begin{equation}
    =  \int_0^1\sum_{\alpha \in \mathcal{A} \setminus \alpha_{min}} \frac{\sum_{i\in \mathcal{D}_\alpha^t} {(\hat{y}_i - y_i)^2}}{|\mathcal{D}_{\alpha}^t|}  dt    
\end{equation}
The first-derivative of \textbf{IDLoss} is taken with regards to a prediction $\hat{y_j}$:
\begin{equation}
    \begin{split}
    \frac{\partial}{\partial \hat{y_j}}\int_0^1 \sum_{\alpha \in \mathcal{A} \setminus \alpha_{min}} \frac{\sum_{i\in \mathcal{D}_\alpha^t} {(\hat{y}_i - y_i)^2}}{|\mathcal{D}_{\alpha}^t|}  dt 
    \end{split}
\end{equation}
\begin{equation}
    \begin{split}
        = \int_0^1\sum_{\alpha \in \mathcal{A} \setminus \alpha_{min}} 
        \frac{2\sum_{i\in \mathcal{D}_{\alpha}^t} {(\hat{y}_i - y_i)}\delta_{ij}}{|\mathcal{D}_{\alpha}^t|} dt
    \end{split}
\end{equation}
where $\delta_{ij}$ is the Kronecker Delta which is 1 when $i = j$ and 0 otherwise. This equation can be rewritten as: 
\begin{equation}
    \begin{split}
        = \int_0^1\sum_{\alpha \in \mathcal{A} \setminus \alpha_{min}}\left. 
        \frac{2 {(\hat{y}_j - y_j)}}{|\mathcal{D}_{\alpha}^t|} \right|_{y_j\in D^t_{\alpha}}dt
    \end{split}
\end{equation}
The second-order derivative of \textbf{IDLoss} with regards to $\hat{y_j}$:
\begin{equation}
    \begin{split}
        \frac{\partial^2 IDLoss}{\partial \hat{y_j}^2} = \int_0^1\sum_{\alpha \in \mathcal{A} \setminus \alpha_{min}}  \frac{2\times\mathbbm{1}(y_j \in D^t_{\alpha})}{|\mathcal{D}_{\alpha}^t|} dt
    \end{split}
\end{equation}
where $\mathbbm{1}(y_j \in D^t_\alpha)$ is an indicator function equal to 1 when $y_j$ is in $D^t_\alpha$ and 0 otherwise.

\subsection{Theoretical Properties}

IDLoss presents unique optimization challenges due to its dependency on $\alpha_{\min}$, which can change during optimization. Despite this complexity, we establish theoretical guarantees for convergence and optimization. A complete theoretical analysis with formal proofs is provided in Appendix~\ref{appendix:properties}.

\begin{itemize}
    \item \textbf{Non-convexity}: IDLoss is non-convex because the identity of $\alpha_{\min}$ can switch during optimization, creating a piecewise structure in the loss landscape. Consider predictions that result in different groups having minimum error---small changes in predictions can cause discrete jumps in which error terms are included in the loss calculation;
    \item \textbf{Convergence Guarantees}: Despite non-convexity, IDLoss satisfies convergence guarantees through the \L{}ojasiewicz inequality~\cite{Lojasiewicz1963TopoProperty}. We partition the prediction space into regions where $\alpha_{\min}$ remains constant. Within each region, IDLoss is analytical, and the \L{}ojasiewicz inequality applies. Since there are finitely many possible values for $\alpha_{\min}$, region transitions are finite, ensuring global convergence to stationary points;
    \item \textbf{Smoothness Properties}: IDLoss has piecewise Lipschitz continuous gradients. Within regions where $\alpha_{\min}$ is constant, the gradient is Lipschitz continuous with computable constants. At region boundaries, bounded discontinuities may occur, but these are finite in number.
\end{itemize}

\begin{table*}[ht] \centering
\centering
\caption{Details regarding each of the datasets used including Name, Prediction Task, Total Number of Samples, Total Number of Features, the Protected Attributes and their respective Privileged Classes, and Intersectional Group Sizes in decreasing order of size.}
\label{tab:datasets}
\resizebox{\textwidth}{!}{
\begin{tabular}{lllllll}
\hline
\textbf{Name} & \textbf{Prediction Task} & \textbf{Cases} & \textbf{Feat.} & \textbf{Protected Attributes} & \textbf{Privileged Classes} & \textbf{Intersectional Group Sizes} \\ \hline
\begin{tabular}[c]{@{}l@{}}Communities \\and Crime\end{tabular} & Violent Crimes per Capita & 1994 & 1971 & \begin{tabular}[c]{@{}l@{}}Percentage of population \\ that is African American\end{tabular} & $<6\%$ & 1024 / 970\\ %\hline
LSAC & Undergraduate GPA & 20802 & 18 & Sex, Race & Male, White & 10098 / 7396 / 1731 / 1577 \\ %\hline
NLSY79 & Total Income (Code: T0912400) & 2341 & 107 & Sex, Race & Male, Non-Black/Non-Hispanic & 722 / 637 / 529 / 453 \\ %\hline
COMPAS & Likelihood of Recidivation Score & 9049 & 13 & Sex, Race & Male, Caucasian & 4813 / 2377 / 1100 / 759 \\ \hline
\end{tabular}
}
\end{table*}

\section{Experimental Evaluation}\label{sec:EE}

Our experimental evaluation aims to answer the following research questions:
\begin{enumerate}[start=1,label={\bfseries RQ\arabic*}]
    \item Is measuring fairness in a single protected attribute sufficient for understanding a model's biases? \label{RQ1}
    \item Can \textbf{ID} be used to visualize fairness and better understand how the model is treating different protected groups w.r.t. domain imbalance? \label{RQ2} %existing error-based metrics?
    \item Can we optimize for \textbf{ID} to build a reliable, fair regression model that is competitive with SOTA baselines? \label{RQ3}
\end{enumerate}
\subsection{Data}
Work in this area is limited by a lack of publicly available fairness datasets for regression tasks. To evaluate the generalizability of our proposal, we used four fairness-oriented, public datasets of varying sizes and protected attributes: Communities and Crime~\cite{misc_communities_and_crime_183}, LSAC~\cite{wightman1998lsac}, NLSY79\footnote{https://www.bls.gov/nls/}, and COMPAS\footnote{https://www.propublica.org/datastore/dataset/compas-recidivism-risk-score-data-and-analysis}. The NLSY79 dataset was collected from the US Bureau of Labor Statistics using the same features as Komiyama, et al.~\cite{pmlr-v80-komiyama18a}. The COMPAS data provided by ProPublica was used to predict a person's ``Likelihood of Recidivism Score" based on demographics and prior arrest history. Pre-processing included removing missing values and dropping non-predictive columns~\cite{le2022survey}. Protected attributes were assigned a binary value. Further details are provided in Table~\ref{tab:datasets}.

\begin{figure}
    \centering
    \includegraphics[width=\linewidth]{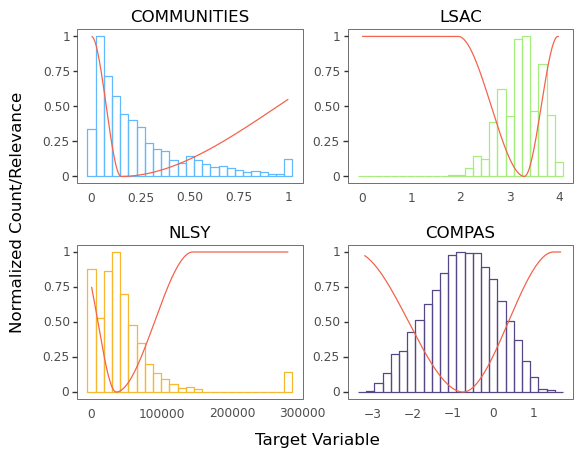}
    \caption{The histograms show the distribution of the target variable in each of the datasets normalized so that the maximum bin has a value of 1. The red line indicates the relevance of each target value interpolated using boxplot statistics.}
    \label{fig:relevance}
\end{figure} 

For each of the datasets, the relevance function was interpolated using boxplot statistics. These functions are visualized in Figure~\ref{fig:relevance} along with the distribution of the target variable. Extreme values on both ends of the distribution typically have a relevance value of 1 while values closer to the median are considered low relevance. 

\subsection{Intersectionality}

To investigate our hypothesis that intersectionality provides critical insights into possible bias in models, we use the LSAC, NLSY, and COMPAS datasets, each containing both Sex and Race protected attributes. In all three datasets, Male and White or Non-Black/Non-Hispanic, respectively, were considered the privileged groups, while Female and Non-White or Black/Hispanic were the unprivileged groups in line with previous work~\cite{le2022survey}. None of the datasets made a distinction for non-binary individuals.
\paragraph{Methodology.} The datasets were split into train and test sets using a train ratio of 80\%, and the former was used to fit XGBoost models.
For each dataset, we calculated the percentage difference in MAE based on race for three groups. First, we considered the overall difference in performance between subgroups for both sex and race across the entire dataset. Then, we compared this to the difference in the errors for each of the intersectional race and sex groups. The results are in Table~\ref{tab:inter-sex}.

\paragraph{Analysis.} Results demonstrate that only considering the difference in error by race hides important biases in the model. In the NLSY dataset, there is a smaller disparity in the treatment of women based on race than there is in the treatment of men. Specifically, if you only look at unfairness by race, the total difference in MAE is 26.8\%. However, for the Male group, the unfairness by race increases to 48.6\% but decreases to 3.3\% for Females. This disparity is overlooked if we only consider a single protected attribute. 

%194, 231, 255
\definecolor{newblue}{rgb}{.76,.906,1}
% 255, 196, 196
\definecolor{newred}{rgb}{1,.769,.769}
\begin{table}[]
\small
\centering
\caption{MAE partitioned by race and sex. Importantly, the final column illustrates how the difference in error changes with sex. Higher absolute differences mean greater unfairness. The sign indicates the direction of the unfairness with positive values indicating lower error for the non-privileged group. Blue indicates a decrease in unfairness from the total group while red indicates an increase.}
\label{tab:inter-sex}
\begin{tabular}{p{.05\linewidth} p{.05\linewidth} p{.05\linewidth} p{.05\linewidth} p{.05\linewidth}}

% \cline{2-4}
\multicolumn{1}{c|}{\multirow{2}{*}{\textbf{LSAC}}} & \multicolumn{2}{c|}{MAE}  & \multicolumn{1}{c}{\multirow{2}{*}{\textit{$\Delta$ \%}}}\\
\multicolumn{1}{c|}{} & \multicolumn{1}{c}{\textit{Race Priv.}} & \multicolumn{1}{c|}{\textit{Race Unpriv.}} & \multicolumn{1}{c}{} \\ \hline
\multicolumn{1}{c|}{\textit{All}}& \multicolumn{1}{c}{0.275} & \multicolumn{1}{c|}{0.320} & \multicolumn{1}{c}{$-14.1\%$}  \\
\multicolumn{1}{c|}{\textit{Male}} & \multicolumn{1}{c}{0.287} & \multicolumn{1}{c|}{0.343} & \multicolumn{1}{c}{\cellcolor{newred!60}$-16.3\%$} \\
\multicolumn{1}{c|}{\textit{Female}} & \multicolumn{1}{c}{0.258} & \multicolumn{1}{c|}{0.296} & \multicolumn{1}{c}{\cellcolor{newblue!40}$-12.8\%$}  \\ 
\\
% \cline{2-4}
\multicolumn{1}{c|}{\multirow{2}{*}{\textbf{NLSY}}} & \multicolumn{2}{c|}{MAE}  & \multicolumn{1}{c}{\multirow{2}{*}{\textit{$\Delta$ \%}}}\\
\multicolumn{1}{c|}{} & \multicolumn{1}{c}{\textit{Race Priv.}} & \multicolumn{1}{c|}{\textit{Race Unpriv.}} & \multicolumn{1}{c}{} \\ \hline
\multicolumn{1}{c|}{\textit{All}} & \multicolumn{1}{c}{22203} & \multicolumn{1}{c|}{17505} & \multicolumn{1}{c}{26.8\%} \\
\multicolumn{1}{c|}{\textit{Male}} & \multicolumn{1}{c}{29277} & \multicolumn{1}{c|}{19707} & \multicolumn{1}{c}{\cellcolor{newred!95}48.6\%} \\
\multicolumn{1}{c|}{\textit{Female}} & \multicolumn{1}{c}{15337} & \multicolumn{1}{c|}{15859} & \multicolumn{1}{c}{\cellcolor{newblue}-3.3\%} \\
\\
% \cline{2-4}
\multicolumn{1}{c|}{\multirow{2}{*}{\textbf{COMPAS}}} & \multicolumn{2}{c|}{MAE}  & \multicolumn{1}{c}{\multirow{2}{*}{\textit{$\Delta$ \%}}}\\
\multicolumn{1}{c|}{} & \multicolumn{1}{c}{\textit{Race Priv.}} & \multicolumn{1}{c|}{\textit{Race Unpriv.}} & \multicolumn{1}{c}{} \\ \hline
\multicolumn{1}{c|}{\textit{All}} & \multicolumn{1}{c}{0.289} & \multicolumn{1}{c|}{0.286} & \multicolumn{1}{c}{1.0\%} \\
\multicolumn{1}{c|}{\textit{Male}} & \multicolumn{1}{c}{0.290} & \multicolumn{1}{c|}{0.294} & \multicolumn{1}{c}{\cellcolor{newred!30}-1.4\%} \\
\multicolumn{1}{c|}{\textit{Female}} & \multicolumn{1}{c}{0.287} & \multicolumn{1}{c|}{0.252} & \multicolumn{1}{c}{\cellcolor{newred}13.9\%} \\ 

\end{tabular}
\end{table}

Additionally, considering a single protected attribute can mislead fairness assessments. In the COMPAS dataset, the overall difference in unfairness by race is positive, i.e., higher error for the race-privileged group than the race-unprivileged group. However, for males, the error is higher for the race-unprivileged group. As a result, \textit{both} male and female groups have higher unfairness by race than the total unfairness indicates. The model appears to have only a small bias in terms of race, but a closer inspection reveals a much larger unfairness problem dependent upon sex.

\paragraph{Conclusion.} Concerning~\ref{RQ1}, we find that measuring fairness in a single protected attribute is insufficient to understand a model's biases. Members of a protected group are frequently treated differently based on their characteristics in other protected attributes. A fairness measure should consider each individual's combination of attributes to ensure no individual subgroup is overlooked within a model. However, this still fails to consider the fact that, in some contexts, some predicted values may be more relevant than others.

\subsection{Domain Imbalance}

Another limitation of current fairness regression measures is that they fail to account for the impact of imbalanced domain preferences, i.e. not all domain values are equally important for users w.r.t. obtaining an accurate prediction. For example, a popular error-based fairness measure is Bounded Group Loss ($\Delta BGL$) proposed by \citet{pmlr-v97-agarwal19d}. $\Delta BGL$ measures the difference in mean absolute error for each group within a single protected attribute without regard for domain imbalance.
\paragraph{Artificial Scenario.}
In Figure~\ref{fig:intro2}, we introduced an artificial scenario in which the various groups have the same total MAE but the MAE for the unprivileged group was disproportionately concentrated in the high-relevance area.

In this example, the $\Delta BGL$ would be 0, indicating perfect fairness. However, when we measure unfairness using \textbf{ID}, as displayed in Figure~\ref{fig:artificialid}, we can see a significant divergence. 
This difference represents unfairness from the imbalanced domain. The error for the privileged group peaked outside the high-relevance area, while the inverse was true for the unprivileged group. As a result, the unprivileged group has a much higher overall error adjusted for relevance than the privileged group, indicating an unfair model.

\begin{figure}
    \centering
    \includegraphics[width=\linewidth]{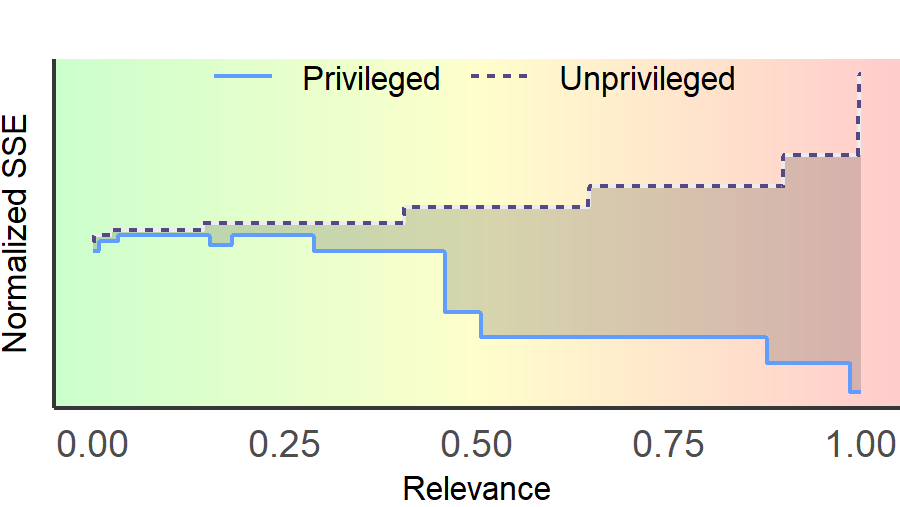}
    \caption{\textbf{ID} graph for the artificial scenario. The x-axis represents the relevance of the predicted values, and the y-axis is the normalized sum of squared error for each group. Using \textbf{ID} we are able to observe the disparate treatment that is overlooked when ignoring domain imbalance.}
    \label{fig:artificialid}
\end{figure} 

\paragraph{Real-World Scenario.} We extend this example, studying the impact of imbalanced domains using real-world datasets, using an XGB model with the LSAC, NLSY79, and COMPAS datasets. We divided the data into train and test splits for each dataset and calculated the average MAE at each predicted value for individual groups of protected attributes. Figure~\ref{fig:real} shows the results.

Similar to the artificial example, the real-world graphs illustrate that performance varies for each protected group at different prediction values. For example, in the NLSY dataset, the Black/Hispanic female and non-Black/non-Hispanic male groups both have large errors on predicted values over 1.5e5, while neither of the other groups extend that far. This disparity illustrates that the model never predicts Black/Hispanic males or non-Black/non-Hispanic females to have a total income above \$150,000. This is indicative of an unfair model but in a way that is not recognizable if you do not consider relevance.

\textbf{ID} addresses this issue by looking at the difference in SERA for each group. Using the \textbf{ID} graphs in  Figure~\ref{fig:real}, we can gain valuable insights into the particular biases of a model. For example, in the NLSY dataset, high total income values were considered to be of high relevance. After normalizing the error, it is evident that the model failed to predict high values for the Black/Hispanic male and the non-Black/non-Hispanic female groups even though these values existed in the ground truths. The \textbf{ID} graphs allow us to visualize this unfair pattern in a way that existing measures do not.

\begin{figure}
    \centering
    \includegraphics[width=\linewidth]{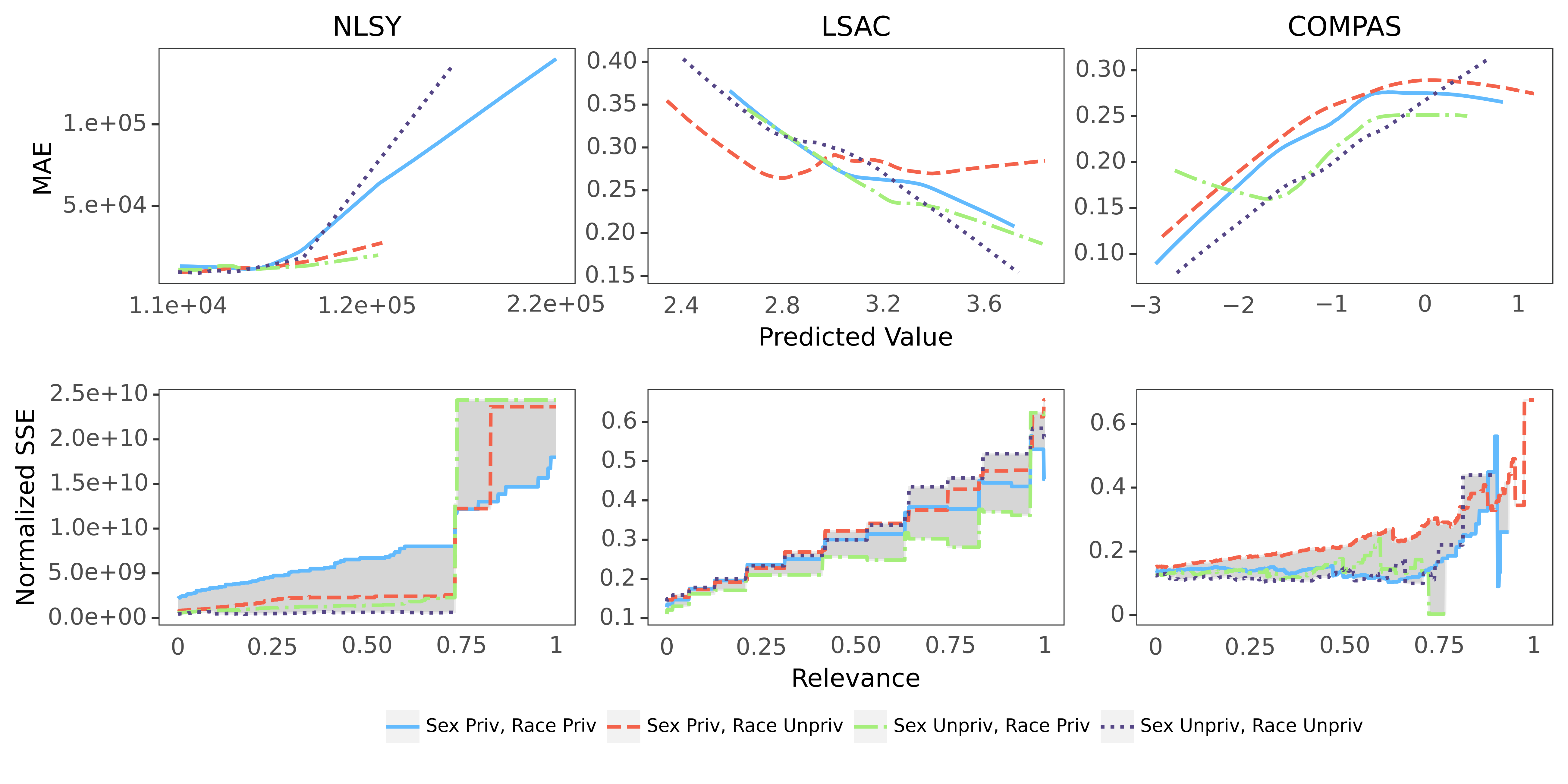}
    \caption{The top row illustrates imbalanced predictions using three real-world datasets where the x-axis depicts the predicted values and the y-axis the Mean Absolute Error. The bottom row shows the corresponding \textbf{ID} graph for each model where the x-axis is the relevance threshold, and the y-axis is the normalized sum of squared error for each group.}
    \label{fig:real}
\end{figure}

\paragraph{Conclusion.} Concerning~\ref{RQ2}, \textbf{ID} allows us to consider the imbalance in our predictions that existing fairness measures neglect. By visualizing a model's results, \textbf{ID} allows a better understanding of how a model behaves unfairly and to identify overlooked biases.

\begin{table*}[] \centering
\caption{Average and Standard Deviation of predictive performance and fairness measures' rankings for all datasets. Algorithms grouped by fairness-agnostic, fairness-aware, and our proposal. Lower numbers indicate better performance. \textbf{Best} and \textit{second-best} results marked.}
\label{tab:full-ranks}
\resizebox{.8\textwidth}{!}{
\begin{tabular}{l|l|cc|ccc}
% \hline
\multicolumn{2}{c|}{} & \multicolumn{2}{c|}{\textbf{Performance Metrics}} & \multicolumn{3}{c}{\textbf{Fairness Metrics}} \\ %\cline{3-7} 
 \multicolumn{2}{c|}{} & \textit{MSE (Avg Rank)} & \textit{SERA (Avg Rank)} & \textit{$\Delta BGL$ (Avg Rank)} & \textit{SP (Avg Rank)} & \textit{ID (Avg Rank)} \\ \hline
\parbox[t]{2mm}{\multirow{4}{*}{\rotatebox[origin=c]{90}{Agnostic}}} & \textbf{XGB$_{MSE}$} & $\mathbf{1.57 \pm 0.85}$ & $3.12 \pm 1.13$ & $6.25 \pm 2.48$ & $7.61 \pm 2.93$ & $4.59 \pm 1.99$ \\
 & \textbf{XGB$_{Huber}$} & $6.58 \pm 4.97$ & $7.47 \pm 4.18$ & $9.61 \pm 2.34$ & $4.78 \pm 4.39$ & $8.54 \pm 3.33$\\
 & \textbf{XGB$_{SERA}$} & $5.99 \pm 2.31$ & $\mathbf{1.56 \pm 1.21}$ & $\mathit{5.30 \pm 3.84}$ & $7.51 \pm 2.38$  & $4.58 \pm 2.97$\\
 & \textbf{XGB$_{Indiv.}$} & $\mathit{2.61 \pm 1.21}$ & $4.50 \pm 1.48$ & $6.39 \pm 2.52$ & $10.70 \pm 1.77$  & $5.81 \pm 2.70$ \\ \hhline{|-|-|--|---|} %\hhline{|=|=|==|===|}
\parbox[t]{2mm}{\multirow{7}{*}{\rotatebox[origin=c]{90}{Aware}}} & \textbf{Agarwal$_{MSE}$}~\cite{pmlr-v97-agarwal19d} & $7.69 \pm 3.63$ & $9.16 \pm 2.23$ & $7.17 \pm 3.33$ & $7.86 \pm 2.89$  & $8.38 \pm 2.96$ \\ 
 & \textbf{Agarwal$_{SERA}$} & $8.00 \pm 3.35$ & $8.82 \pm 2.41$ & $6.91 \pm 3.44$ & $7.76 \pm 2.92$ & $8.36 \pm 3.00$  \\
 & \textbf{Agarwal$_{ID}$} & $7.90 \pm 3.43$ & $9.18 \pm 2.18$ & $6.80 \pm 3.53$ & $8.18 \pm 2.80$ & $8.06 \pm 3.23$ \\
 & \textbf{Calders$_{\alpha = 0}$}~\cite{calders2013controlling} & $8.10 \pm 3.64$ & $8.47 \pm 3.29$ & $6.92 \pm 4.02$ & $8.93 \pm 2.89$ & $7.95 \pm 3.17$ \\
 & \textbf{Calders$_{\alpha = 5}$} & $7.97 \pm 3.65$ & $9.24 \pm 3.13$ & $6.49 \pm 3.64$ & $8.16 \pm 2.88$ & $7.46 \pm 3.23$ \\
& \textbf{P\'erez-Suay$_{1NN}$}~\cite{perez2017fair} & $10.07 \pm 2.29$ & $10.11 \pm 2.30$ & $9.47 \pm 3.89$ & $\mathbf{2.42 \pm 1.72}$ & $10.70 \pm 2.65$ \\
& \textbf{P\'erez-Suay$_{XGB}$} & $8.55 \pm 2.19$ & $9.06 \pm 2.09$ & $7.70 \pm 4.52$ & $7.50 \pm 4.93$ & $9.69 \pm 3.14$ \\ \hhline{|-|-|--|---|} %\hhline{|=|=|==|===|}
\parbox[t]{2mm}{\multirow{2}{*}{\rotatebox[origin=c]{90}{Ours}}} 
 & \textbf{IDBoost$_{0.5}$} & $6.76 \pm 1.63$ & $\mathit{2.89 \pm 1.48}$ & $\mathbf{3.92 \pm 2.78}$ & $6.26 \pm 2.78$ & $\mathbf{3.34 \pm 2.67}$ \\
 & \textbf{IDBoost$_{1.0}$} & $9.20 \pm 1.75$ & $7.40 \pm 3.02$ & $8.05 \pm 3.67$ & $\mathit{3.34 \pm 2.52}$ & $\mathit{3.55 \pm 3.23}$ \\ 
 % \hline 
\end{tabular}}
\end{table*}

\subsection{Evaluation of IDLoss}

Next, we demonstrate how \textbf{ID} combined with an optimization technique can build a fairness-aware regression model. The goal is to minimize the disparity between all pairs of protected attributes while minimizing overall error.

In this section, we use a general in-processing framework demonstrating how \textbf{IDLoss} and SERA can be used with a boosting technique to improve model fairness while retaining predictive performance. We refer to this framework as IDBoost. We implement IDBoost using the XGBoost (XGB) algorithm~\cite{10.1145/2939672.2939785} to demonstrate the effectiveness of IDLoss in optimization. Importantly, IDLoss can be adapted for any algorithm using a loss function.

IDBoost is trained using two ensembles of decision trees. Optimized for fairness, the first ensemble weights samples based on the learner's performance measured by \textbf{IDLoss}. Optimized for predictive performance, the other ensemble weights samples based on performance with SERA. Then, the two ensembles' predictions are averaged with user-specified fairness/predictive weights.

\paragraph{Methodology.} We compare IDBoost against state-of-the-art fairness regression solutions in prediction and fairness performance. We use MSE and SERA to measure predictive error and $\Delta BGL$, Statistical Parity (SP), and \textbf{ID} to measure fairness~\cite{pmlr-v97-agarwal19d}. SP measures the difference in CDF for groups in a single protected attribute. Unlike \textbf{ID}, SP does not consider the true value of the sample. For $\Delta BGL$ and SP, which only measure one protected attribute at a time, the model was scored using the average across all protected attributes. We used 20 different train and test splits for all four datasets with a train ratio of 80\%. 

We measure our proposal against three state-of-the-art solutions. The first, proposed by ~\citet{calders2013controlling}, optimizes around the mean difference between predictions. The next, proposed by ~\citet{perez2017fair}, is a pre-processing method that reduces a dataset to a single, fair dimension. We evaluated this algorithm using a 1-Nearest Neighbor algorithm, as in their original paper, and an XGB model optimized for MSE. 
The final solution, proposed by ~\citet{pmlr-v97-agarwal19d}, combines linear regression with additional fairness constraints. The algorithm ensures that the Bounded Group Loss is less than a user-specified threshold. Going forward, we refer to these solutions by the author's name.

Agarwal can only be optimized for a single attribute at a time. To provide the fairest comparison, we evaluated multiple Agarwal models, one optimized for each protected attribute in a given dataset. Then, we picked the best-performing model for a given metric on the test set -- denoted Agarwal$_{\{metric\}}$ in our results. We also trained XGB models for each set of protected attributes. These aim to minimize the overall error for each group separately. This set of models is labeled XGB$_{Indiv.}$. Finally, we compared our solution against three fairness-agnostic XGB models optimized using MSE, Huber, and SERA loss functions.

We tested two different versions of the IDBoost framework. IDBoost$_{1.0}$ tests the performance of our algorithm using only \textbf{IDLoss} boosting. IDBoost$_{0.5}$ combines the performance of our \textbf{IDLoss} boosting and the SERA boosting techniques.
The models were ranked for each run by each metric performance. Models that failed to run received a rank of last place. Ranks were averaged across all 80 trials (4 datasets, 20 runs each). Full results are available in Appendix~\ref{app:res}.
\paragraph{Analysis.} The results from the experiments are found in Table~\ref{tab:full-ranks}. Overall, the IDBoost algorithm combining \textbf{IDLoss} and SERA optimization with a 50\% weight on each is the best fairness-aware algorithm regardless of the performance measure. Additionally, while it lags behind each XGB model in MSE, it is second only to XGB$_{SERA}$ in SERA. It is better than all XGB models in both fairness measures. XGB$_{SERA}$ is most competitive with our proposal but worst in both the existing and proposed fairness measures. IDBoost does best in recognizing and correcting intersectional unfairness and imbalanced predictions. 

\paragraph{Conclusion.} Concerning~\ref{RQ3}, results show that \textbf{ID} can be used within a regression model and improve upon SOTA baselines w.r.t. both fairness and predictive measures.

\section{Discussion}

A main advantage of \textbf{ID} is its unique ability to visualize the results and gain a deeper understanding of a model's unfair behavior. To showcase \textbf{ID}'s ability to provide insights in a real-world setting, we present Figure~\ref{fig:avgid} where the \textbf{ID} curves for two datasets are averaged for three competing solutions: XGB$_{SERA}$, IDBoost$_{0.5}$, and IDBoost$_{1.0}$.

\begin{figure}
    \centering
    \includegraphics[width=\linewidth]{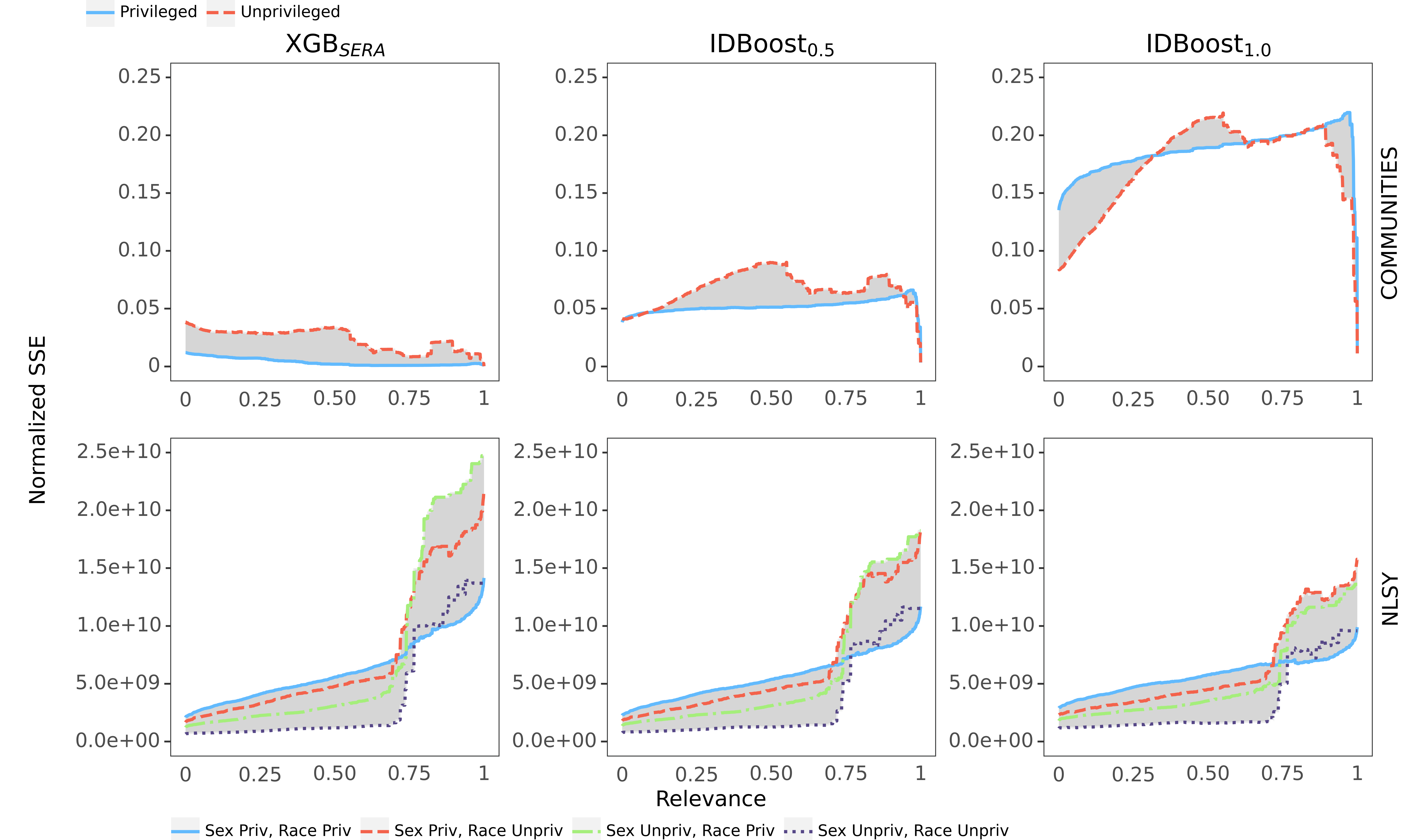}
    \caption{Average \textbf{ID} across all 20 runs. The Communities dataset has one protected attribute while the NLSY dataset has two.}
    \label{fig:avgid}
\end{figure}

Comparing the \textbf{ID} graphs, we can clearly understand why IDBoost$_{1.0}$ is fairer than $XGB_{SERA}$. For example, on NLSY, XGB$_{SERA}$ is best at predicting low-relevance values and has the smallest divergence at 0 relevance (i.e. the total error when considering all predictions). However, $XGB_{SERA}$ is much worse at predicting high-relevance values for the White female group than IDBoost$_{1.0}$ and as a result has a worse \textbf{ID}. 

Furthermore, with these graphs we can better understand the strong performance of IDBoost$_{0.5}$ from above. In Communities, XGB$_{SERA}$ is best at minimizing the total error and performs better for the privileged group at every relevance threshold. Meanwhile, IDBoost$_{1.0}$ performs better at predicting the unprivileged group for most of the low- and high-relevance thresholds. IDBoost$_{0.5}$ effectively combines the models, minimizing the total divergence while limiting the predictive performance trade-off.

The main challenge in our proposal centers around the number of protected attributes. As we increase the number of protected attributes, the number of samples in each group decreases substantially while the runtime grows exponentially. In our view, these limitations are not prohibitive because the number of protected attributes is typically small. Nonetheless, as this is one of the first proposals to incorporate intersectionality in a regression setting, we envision future work seeking to address these issues. Small samples may be addressed through traditional data imbalance techniques such as oversampling. Meanwhile, efficiency can be improved by including approximation techniques during optimization.

\begin{figure}
    \centering
    \includegraphics[width=\linewidth]{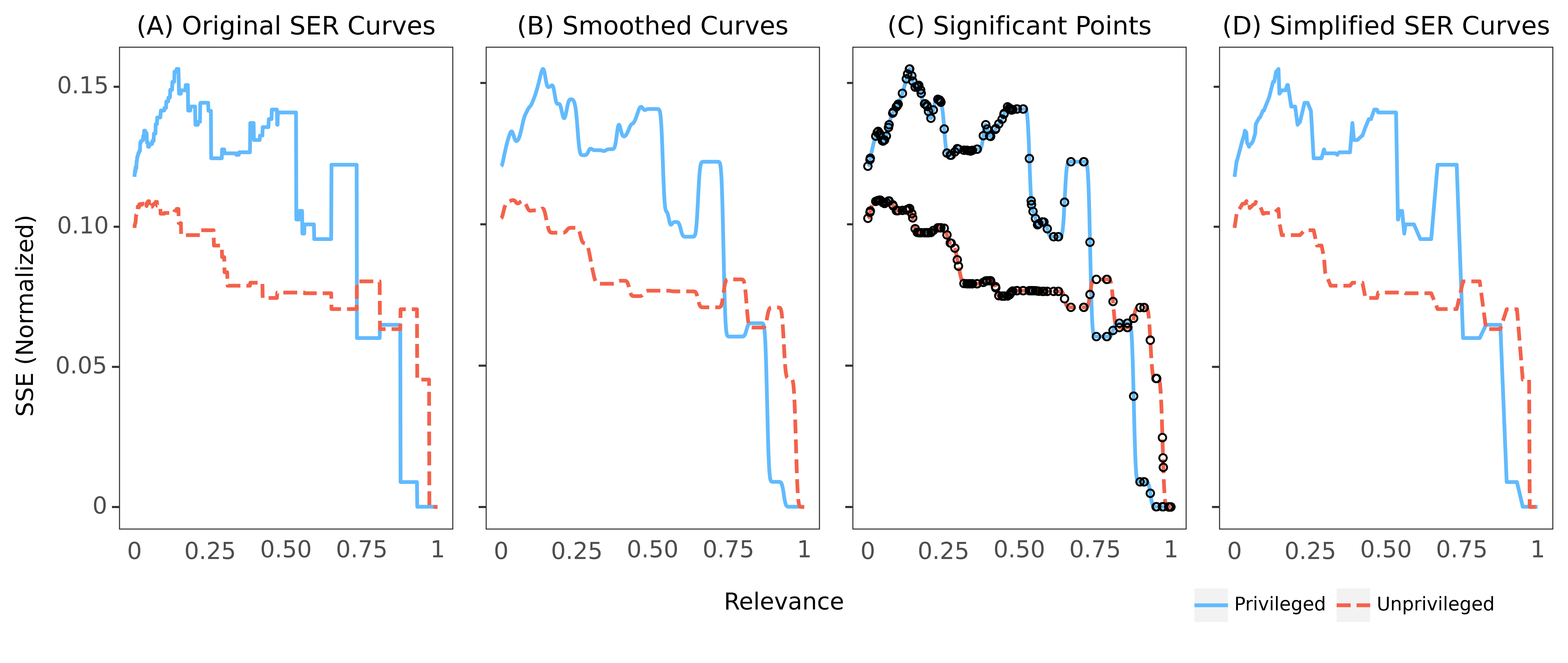}
    \caption{Overview of the process to approximate SER curves. (A) Original SER curves. (B) Apply Gaussian smoothing. (C) Identify points where the first or second derivative equals 0. (D) Approximate SER curves using the points found in (C). Data from a Linear Regression model on the Communities dataset.}
    \label{fig:approx}
\end{figure}

As a demonstration, we introduce a simple strategy which can significantly improve runtime with minimal performance degradation illustrated in Figure~\ref{fig:approx}. We calculate the original SER curves and apply Gaussian smoothing to approximate each. We then identify the ``significant points'' by finding each value along the line where the first or second derivative is equal to zero. Finally, we redraw simplified SER curves using only these ``significant points''. This procedure achieves an accurate approximation of the original SER curves while using fewer points along the x-axis. These new curves are used to calculate the errors targeting one of the main bottlenecks in the SERA and IDLoss algorithms. 

\begin{figure}
    \centering
    \includegraphics[width=\linewidth]{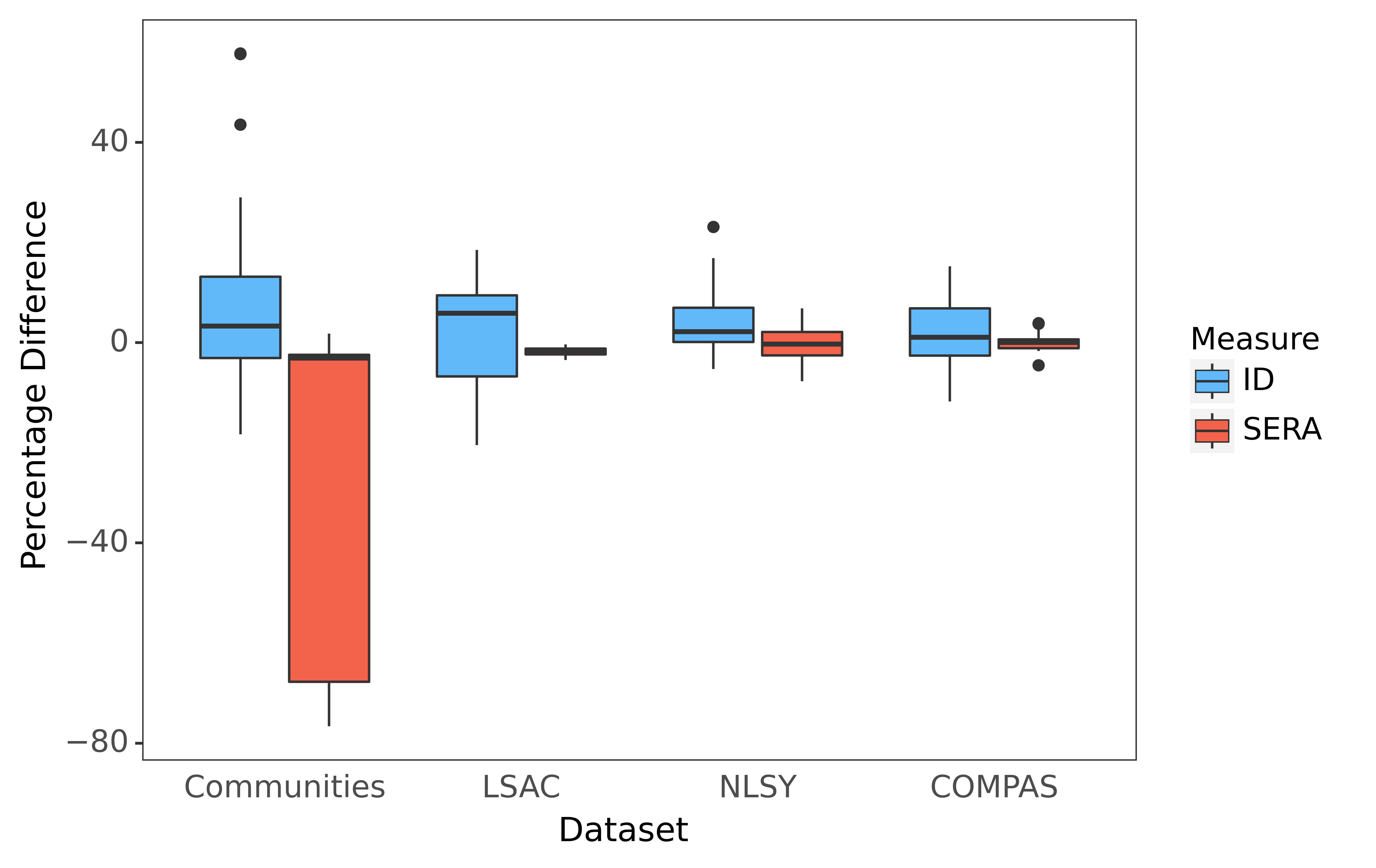}
    \\ \hfill \\
\resizebox{\linewidth}{!}{\begin{tabular}{|l|cccc|}
\hline

\textbf{Time (s)} & \textbf{Communities} & \textbf{LSAC} & \textbf{NLSY} & \textbf{COMPAS} \\
\hline
\textbf{IDBoost$_{0.5}$} & 4976.3 $\pm$ 58.4 & 45716.6 $\pm$ 688.4 & 6966.1 $\pm$ 56.5 & 19844.8 $\pm$ 223.0 \\
\textbf{IDBoost$_{0.5} (FAST)$ } & 2546.8 $\pm$ 36.8 & 25130.3 $\pm$ 1230.4 & 4354.6 $\pm$ 123.7 & 13581.2 $\pm$ 208.2 \\
\hline
\textbf{Percentage Difference} & \textbf{-48.8\%} & \textbf{-45.0\%} & \textbf{-37.5\%} & \textbf{-31.6\%} \\
\hline
\end{tabular}}
\caption{Comparison of performance and runtime between IDBoost with and without the approximation procedure. The box plots illustrate the percentage difference in SERA and ID between the two models across all 20 runs of the 4 datasets. The table provides the average and standard deviation of the processing time required to train and predict each algorithm.}
\label{fig:fast}
\end{figure}

As Figure~\ref{fig:fast} shows, incorporating this approximation technique can decrease processing time by greater than 30\% without a significant change in SERA or ID in 3 of the 4 datasets. Future work will investigate further ways to decrease processing time without a significant drop in performance.

\section{Conclusion}
We propose a new method for measuring fairness in regression tasks. 
Our measure improves upon existing fairness measures by being the first to i) consider the intersectionality of multiple protected attributes and ii) address the need to have a predictive focus on certain ranges of values in imbalanced domains, allowing for more robust fairness considerations and ensuring that all subgroups are better represented.
Additionally, our approach is able to visualize the differences in fairness, making it easier to understand and address the areas of weakness within a model.
Finally, we demonstrate that a dual boosting approach using \textbf{ID} alongside a performance measure such as SERA creates a fair regression model that improves fairness while maintaining strong predictive performance. From a theoretical perspective, we provide the first rigorous analysis of convergence properties for intersectional fairness optimization in regression. Our analysis establishes that despite IDLoss being non-convex, it satisfies the \L{}ojasiewicz inequality ensuring convergence to stationary points, and has piecewise smooth properties enabling practical optimization. These theoretical foundations explain the empirical success of our methods and provide guidance for future algorithm development in fair regression. We make all the code available for reproducibility purposes at \url{https://anonymous.4open.science/r/ID-8A60/}.

\bibliographystyle{ACM-Reference-Format}
\bibliography{bib}

\appendix

\renewcommand{\thetable}{B\arabic{table}}
\setcounter{table}{0}

\newpage
\section{Appendix: Theoretical Analysis of IDLoss}\label{appendix:properties}

This appendix provides a comprehensive theoretical analysis of the Intersectional Divergence Loss function (IDLoss), establishing its fundamental mathematical properties and convergence guarantees. Our analysis addresses three key aspects: (1) the non-convex nature of the optimization landscape, (2) convergence guarantees despite non-convexity, and (3) the smoothness properties that enable practical optimization.

\subsection{Mathematical Preliminaries}

\begin{definition}[IDLoss]
\label{def:idloss}
Given protected attributes $\mathcal{A}$ with all possible combinations $A$, the IDLoss function is defined as:
\begin{equation}
\text{IDLoss} = \int_0^1 \sum_{\alpha \in A \setminus \alpha_{\min}} \frac{\sum_{i \in D_t^\alpha} (\hat{y}_i - y_i)^2}{|D_t^\alpha|} dt
\end{equation}
where $\alpha_{\min} = \arg\min_{\alpha \in A} \frac{\sum_{i \in D_t^\alpha} (\hat{y}_i - y_i)^2}{|D_t^\alpha|}$ is the protected attribute combination with minimum normalized error at relevance $t$.
\end{definition}

\begin{definition}[Region Partition]
\label{def:region_partition}
The prediction space $\mathbb{R}^n$ can be partitioned into regions $\{R_k\}_{k=1}^K$ where:
\begin{equation}
R_k = \left\{\hat{\mathbf{y}} \in \mathbb{R}^n : \arg\min_{\alpha \in A} \frac{\sum_{i \in D_t^\alpha} (\hat{y}_i - y_i)^2}{|D_t^\alpha|} = \alpha_k \text{ for all } t \in [0,1]\right\}
\end{equation}
Within each region $R_k$, the identity of $\alpha_{\min}$ remains constant, making IDLoss analytical.
\end{definition}

\subsection{Non-Convexity Analysis}

\subsubsection{Demonstration of Non-Convexity}

\begin{proposition}
\label{prop:nonconvex}
IDLoss is non-convex.
\end{proposition}

\begin{proof}
We construct a counterexample that violates the convexity condition $f(\lambda x + (1-\lambda)y) \leq \lambda f(x) + (1-\lambda)f(y)$ for $\lambda \in (0,1)$.

Consider a dataset with two protected attribute combinations $(\alpha_1, \alpha_2)$ and four samples:
\begin{itemize}
\item For $\alpha_1$: Sample 1 with $y_1 = 1$, Sample 2 with $y_2 = 2$
\item For $\alpha_2$: Sample 3 with $y_3 = 3$, Sample 4 with $y_4 = 4$
\end{itemize}

\textbf{Case 1}: Predictions $\hat{\mathbf{y}}^A = [1.2, 2.2, 3.3, 3.9]$
\begin{align}
\text{Error for } \alpha_1 &= \frac{(1.2-1)^2 + (2.2-2)^2}{2} = 0.04 \\
\text{Error for } \alpha_2 &= \frac{(3.3-3)^2 + (3.9-4)^2}{2} = 0.05
\end{align}
Since $\alpha_1$ has minimum error, $\text{IDLoss}(\hat{\mathbf{y}}^A) = 0.05$.

\textbf{Case 2}: Predictions $\hat{\mathbf{y}}^B = [0.8, 1.8, 2.7, 4.1]$
\begin{align}
\text{Error for } \alpha_1 &= \frac{(0.8-1)^2 + (1.8-2)^2}{2} = 0.04 \\
\text{Error for } \alpha_2 &= \frac{(2.7-3)^2 + (4.1-4)^2}{2} = 0.05
\end{align}
Since $\alpha_1$ has minimum error, $\text{IDLoss}(\hat{\mathbf{y}}^B) = 0.05$.

\textbf{Convex Combination}: $\hat{\mathbf{y}}^C = 0.5\hat{\mathbf{y}}^A + 0.5\hat{\mathbf{y}}^B = [1.0, 2.0, 3.0, 4.0]$
\begin{align}
\text{Error for } \alpha_1 &= \frac{(1.0-1)^2 + (2.0-2)^2}{2} = 0 \\
\text{Error for } \alpha_2 &= \frac{(3.0-3)^2 + (4.0-4)^2}{2} = 0
\end{align}
With perfect predictions, $\text{IDLoss}(\hat{\mathbf{y}}^C) = 0$.

Therefore: 
\begin{equation}
\text{IDLoss}(0.5\hat{\mathbf{y}}^A + 0.5\hat{\mathbf{y}}^B) = 0 < 0.5 \cdot 0.05 + 0.5 \cdot 0.05 = 0.05
\end{equation}

This violates convexity, establishing that IDLoss is non-convex.
\end{proof}

\subsubsection{Structural Analysis of Non-Convexity}

\begin{lemma}
\label{lemma:nonconvex_sources}
The non-convexity of IDLoss arises from two sources:
\begin{enumerate}
\item The dependency on $\alpha_{\min}$, which changes during optimization
\item Discontinuities in the gradient at region boundaries
\end{enumerate}
\end{lemma}

\begin{proof}
Within each region $R_k$ where $\alpha_{\min}$ is constant, IDLoss reduces to:
\begin{equation}
\text{IDLoss}|_{R_k} = \int_0^1 \sum_{\alpha \in A \setminus \alpha_k} \frac{\sum_{i \in D_t^\alpha} (\hat{y}_i - y_i)^2}{|D_t^\alpha|} dt
\end{equation}

This is a weighted sum of convex squared error terms, hence convex within $R_k$. The non-convexity emerges from:
\begin{itemize}
\item \textbf{Switching behavior}: As optimization progresses, a different group may become $\alpha_{\min}$, causing a discrete change in the loss function
\item \textbf{Boundary discontinuities}: At region boundaries, the gradient can have jump discontinuities
\end{itemize}
\end{proof}

\subsection{Convergence Analysis via Łojasiewicz Inequality}

\subsubsection{Background on Łojasiewicz Inequality}

The Łojasiewicz inequality provides convergence guarantees for non-convex optimization problems. For a function $f$ that is analytical in a neighborhood of a critical point $x^*$, there exist constants $c > 0$ and $\theta \in [0, 1)$ such that:
\begin{equation}
|f(x) - f(x^*)|^\theta \leq c \|\nabla f(x)\|
\end{equation}

\subsubsection{Main Convergence Result}

\begin{theorem}
\label{thm:convergence}
The gradient descent algorithm applied to IDLoss converges to a stationary point despite its non-convexity.
\end{theorem}

\begin{proof}
Our proof strategy partitions the analysis by regions and applies the Łojasiewicz inequality within each region.

\textbf{Step 1: Regional Analysis}
Within each region $R_k$, IDLoss is analytical as it consists of smooth squared error terms. For analytical functions on compact domains, the Łojasiewicz inequality holds with constants $c_k > 0$ and $\theta_k \in [0, 1)$.

\textbf{Step 2: Global Constants}
Define global constants:
\begin{align}
\theta &= \min_k \theta_k \quad \text{(most restrictive exponent)} \\
c &= \max_k c_k \quad \text{(least favorable constant)}
\end{align}

\textbf{Step 3: Gradient Descent Dynamics}
For the gradient descent sequence $\{\hat{\mathbf{y}}^{(t)}\}$ with step size $\eta_t$:
\begin{equation}
\hat{\mathbf{y}}^{(t+1)} = \hat{\mathbf{y}}^{(t)} - \eta_t \nabla \text{IDLoss}(\hat{\mathbf{y}}^{(t)})
\end{equation}

Within each region $R_k$, standard Łojasiewicz convergence results apply:
\begin{itemize}
\item If $\theta \in [0, \frac{1}{2}]$: Finite-time convergence to stationary point
\item If $\theta \in (\frac{1}{2}, 1)$: Convergence rate $O(t^{-\frac{1}{1-2\theta}})$
\end{itemize}

\textbf{Step 4: Handling Region Transitions}
Each region boundary crossing reduces IDLoss by at least $\delta > 0$ (since switching $\alpha_{\min}$ improves the minimum error). Since IDLoss is bounded below by 0, the number of region crossings is finite.

\textbf{Step 5: Global Convergence}
With finitely many region crossings and guaranteed convergence within each region, the overall sequence converges to a stationary point.
\end{proof}

\begin{corollary}
\label{cor:lojasiewicz}
Despite non-convexity, IDLoss satisfies the Łojasiewicz inequality globally, guaranteeing convergence of gradient-based methods to stationary points.
\end{corollary}

\subsubsection{Computational Complexity}

\begin{theorem}
\label{thm:complexity}
IDBoost achieves an $\varepsilon$-stationary point (i.e., $\|\nabla \text{IDLoss}\| \leq \varepsilon$) in $O(\varepsilon^{-\frac{2}{1-2\theta}})$ iterations with appropriate step size selection.
\end{theorem}

\begin{proof}
This follows directly from applying standard Łojasiewicz convergence rate analysis to our setting, combined with the finite region crossing argument. Each gradient computation requires $O(n|A|)$ operations, where $n$ is the number of samples and $|A|$ is the number of protected attribute combinations.
\end{proof}

\subsection{Smoothness Properties}

\subsubsection{Piecewise Lipschitz Continuity}

\begin{theorem}
\label{thm:lipschitz}
IDLoss has a piecewise Lipschitz continuous gradient, with Lipschitz continuity holding within each region where $\alpha_{\min}$ is constant, and bounded discontinuities at region boundaries.
\end{theorem}

\begin{proof}
Within region $R_k$, the gradient with respect to prediction $\hat{y}_j$ is:
\begin{equation}
\frac{\partial \text{IDLoss}}{\partial \hat{y}_j} = \int_0^1 \sum_{\alpha \in A \setminus \alpha_k} \frac{2(\hat{y}_j - y_j)}{|D_t^\alpha|} \cdot \mathbf{1}(y_j \in D_t^\alpha) dt
\end{equation}

For two prediction vectors $\hat{\mathbf{y}}, \hat{\mathbf{y}}'$ within $R_k$:
\begin{equation}
\left|\frac{\partial \text{IDLoss}}{\partial \hat{y}_j}(\hat{\mathbf{y}}) - \frac{\partial \text{IDLoss}}{\partial \hat{y}_j}(\hat{\mathbf{y}}')\right| \leq C_j |\hat{y}_j - \hat{y}_j'|
\end{equation}

where:
\begin{equation}
C_j = \int_0^1 \sum_{\alpha \in A \setminus \alpha_k} \frac{2}{|D_t^\alpha|} \cdot \mathbf{1}(y_j \in D_t^\alpha) dt
\end{equation}

Taking the norm across all components:
\begin{equation}
\|\nabla \text{IDLoss}(\hat{\mathbf{y}}) - \nabla \text{IDLoss}(\hat{\mathbf{y}}')\| \leq L_k \|\hat{\mathbf{y}} - \hat{\mathbf{y}}'\|
\end{equation}

where $L_k = \sqrt{\sum_j C_j^2}$ is the Lipschitz constant for region $R_k$.

At region boundaries, the gradient may have bounded jump discontinuities due to the discrete change in $\alpha_{\min}$, but these are finite in number and magnitude.
\end{proof}

\subsubsection{Practical Implications for Optimization}

\begin{proposition}
\label{prop:optimization_implications}
The piecewise Lipschitz property enables practical optimization algorithms with the following guarantees:
\begin{enumerate}
\item Within each region, standard gradient-based methods apply with Lipschitz constant $L_k$
\item Appropriate step size selection: $\eta \leq \frac{1}{L_k}$ ensures monotonic improvement within regions
\item Region transitions correspond to discrete improvements in the objective
\end{enumerate}
\end{proposition}

\section{Full Results} \label{app:res}
This section contains results detailing each algorithm's performance across the individual datasets. These results were aggregated to compute the average rankings presented in the section~\ref{sec:EE}.
\begin{table*}[b!]
        \centering
        \caption{Detailed results for \textbf{COMMUNITIES} Dataset}
        \label{atab:communities}
        \resizebox{.8\textwidth}{!}{
        \begin{tabular}{l|l|cc|ccc}
        \multicolumn{2}{c|}{} & \multicolumn{2}{c|}{\textbf{Performance Metrics}} & \multicolumn{3}{c}{\textbf{Fairness Metrics}} \\ 
         \multicolumn{2}{c|}{} & \textit{MSE (Avg Rank)} & \textit{SERA (Avg Rank)} & \textit{$\Delta BGL$ (Avg Rank)} & \textit{SP (Avg Rank)} & \textit{ID (Avg Rank)} \\ \hline
        \parbox[t]{2mm}{\multirow{4}{*}{\rotatebox[origin=c]{90}{Agnostic}}} & \textbf{XGB$_{MSE}$} & $0.02 \pm 0.00$ & $2.03 \pm 0.38$ & $0.06 \pm 0.01$ & $0.27 \pm 0.02$ & $28.74 \pm 10.64$ \\
         & \textbf{XGB$_{Huber}$} & $0.02 \pm 0.00$ & $2.06 \pm 0.35$ & $0.06 \pm 0.01$ & $0.27 \pm 0.02$ & $30.07 \pm 9.39$\\
         & \textbf{XGB$_{SERA}$} & $0.03 \pm 0.00$ & $1.62 \pm 0.33$ & $0.08 \pm 0.01$ & $0.25 \pm 0.02$  & $19.77 \pm 10.72$\\
         & \textbf{XGB$_{Indiv.}$} & $0.02 \pm 0.00$ & $2.08 \pm 0.36$ & $0.07 \pm 0.01$ & $0.29 \pm 0.02$  & $31.72 \pm 10.42$ \\ \hhline{|-|-|--|---|}
        \parbox[t]{2mm}{\multirow{7}{*}{\rotatebox[origin=c]{90}{Aware}}} & \textbf{Agarwal$_{MSE}$} & $0.20 \pm 0.08$ & $23.80 \pm 10.13$ & $0.03 \pm 0.02$ & $0.12 \pm 0.04$  & $81.10 \pm 35.41$ \\ 
         & \textbf{Agarwal$_{SERA}$} & $0.20 \pm 0.08$ & $23.80 \pm 10.13$ & $0.03 \pm 0.02$ & $0.12 \pm 0.04$ & $81.10 \pm 35.41$  \\
         & \textbf{Agarwal$_{ID}$} & $0.20 \pm 0.08$ & $23.80 \pm 10.13$ & $0.03 \pm 0.02$ & $0.12 \pm 0.04$ & $81.10 \pm 35.41$ \\
         & \textbf{Calders$_{\alpha = 0}$} & $0.21 \pm 0.08$ & $24.94 \pm 10.20$ & $0.03 \pm 0.02$ & $0.12 \pm 0.04$ & $86.24 \pm 41.56$ \\
         & \textbf{Calders$_{\alpha = 5}$} & $0.21 \pm 0.08$ & $24.94 \pm 10.19$ & $0.03 \pm 0.02$ & $0.12 \pm 0.04$ & $86.22 \pm 41.51$ \\
        & \textbf{P'erez-Suay$_{1NN}$}~ & $0.10 \pm 0.12$ & $15.77 \pm 17.28$ & $0.13 \pm 0.07$ & $0.03 \pm 0.07$ & $139.69 \pm 44.68$ \\
        & \textbf{P'erez-Suay$_{XGB}$} & $0.06 \pm 0.01$ & $10.58 \pm 2.67$ & $0.10 \pm 0.04$ & $0.03 \pm 0.06$ & $111.52 \pm 35.28$ \\ \hhline{|-|-|--|---|}
        \parbox[t]{2mm}{\multirow{2}{*}{\rotatebox[origin=c]{90}{Ours}}} 
         & \textbf{IDBoost$_{0.5}$} & $0.04 \pm 0.00$ & $6.71 \pm 0.49$ & $0.01 \pm 0.00$ & $0.25 \pm 0.02$ & $17.64 \pm 7.14$ \\
         & \textbf{IDBoost$_{1.0}$} & $0.11 \pm 0.00$ & $21.37 \pm 1.54$ & $0.10 \pm 0.01$ & $0.00 \pm 0.00$ & $19.04 \pm 3.27$ \\
        \end{tabular}}
    \end{table*}

\begin{table*}[] \centering
        \caption{Detailed results for \textbf{LSAC} Dataset}
        \label{atab:lsac}
        \resizebox{.8\textwidth}{!}{
        \begin{tabular}{l|l|cc|ccc}
        \multicolumn{2}{c|}{} & \multicolumn{2}{c|}{\textbf{Performance Metrics}} & \multicolumn{3}{c}{\textbf{Fairness Metrics}} \\ 
         \multicolumn{2}{c|}{} & \textit{MSE (Avg Rank)} & \textit{SERA (Avg Rank)} & \textit{$\Delta BGL$ (Avg Rank)} & \textit{SP (Avg Rank)} & \textit{ID (Avg Rank)} \\ \hline
        \parbox[t]{2mm}{\multirow{4}{*}{\rotatebox[origin=c]{90}{Agnostic}}} & \textbf{XGB$_{MSE}$} & $0.12 \pm 0.00$ & $280.63 \pm 9.47$ & $0.03 \pm 0.01$ & $0.08 \pm 0.00$ & $130.48 \pm 51.58$ \\
         & \textbf{XGB$_{Huber}$} & $916.87 \pm 3.63$ & $1165027.18 \pm 20385.58$ & $0.16 \pm 0.01$ & $0.00 \pm 0.00$ & $44708.17 \pm 6439.49$\\
         & \textbf{XGB$_{SERA}$} & $0.14 \pm 0.00$ & $250.99 \pm 11.99$ & $0.02 \pm 0.01$ & $0.08 \pm 0.00$  & $153.96 \pm 63.97$\\
         & \textbf{XGB$_{Indiv.}$} & $0.12 \pm 0.00$ & $288.48 \pm 10.37$ & $0.03 \pm 0.01$ & $0.09 \pm 0.00$  & $132.12 \pm 49.01$ \\ \hhline{|-|-|--|---|}
        \parbox[t]{2mm}{\multirow{7}{*}{\rotatebox[origin=c]{90}{Aware}}} & \textbf{Agarwal$_{MSE}$} & $0.13 \pm 0.00$ & $306.66 \pm 12.03$ & $0.03 \pm 0.00$ & $0.08 \pm 0.01$  & $137.16 \pm 48.60$ \\ 
         & \textbf{Agarwal$_{SERA}$} & $0.13 \pm 0.00$ & $305.20 \pm 10.71$ & $0.03 \pm 0.00$ & $0.08 \pm 0.00$ & $137.49 \pm 47.96$  \\
         & \textbf{Agarwal$_{ID}$} & $0.13 \pm 0.00$ & $308.83 \pm 13.03$ & $0.03 \pm 0.00$ & $0.08 \pm 0.01$ & $133.93 \pm 46.92$ \\
         & \textbf{Calders$_{\alpha = 0}$} & $0.14 \pm 0.00$ & $313.48 \pm 10.00$ & $0.03 \pm 0.01$ & $0.09 \pm 0.01$ & $141.78 \pm 49.32$ \\
         & \textbf{Calders$_{\alpha = 5}$} & $0.14 \pm 0.00$ & $313.51 \pm 10.00$ & $0.03 \pm 0.01$ & $0.09 \pm 0.01$ & $141.69 \pm 49.25$ \\
        & \textbf{P'erez-Suay$_{1NN}$}~ & $0.29 \pm 0.04$ & $568.33 \pm 53.64$ & $0.05 \pm 0.01$ & $0.04 \pm 0.01$ & $301.85 \pm 68.50$ \\
        & \textbf{P'erez-Suay$_{XGB}$} & $0.16 \pm 0.00$ & $424.54 \pm 14.66$ & $0.04 \pm 0.01$ & $0.10 \pm 0.03$ & $271.02 \pm 73.64$ \\ \hhline{|-|-|--|---|}
        \parbox[t]{2mm}{\multirow{2}{*}{\rotatebox[origin=c]{90}{Ours}}} 
         & \textbf{IDBoost$_{0.5}$} & $0.14 \pm 0.01$ & $268.34 \pm 11.40$ & $0.03 \pm 0.01$ & $0.08 \pm 0.01$ & $147.66 \pm 61.48$ \\
         & \textbf{IDBoost$_{1.0}$} & $0.16 \pm 0.02$ & $322.19 \pm 31.44$ & $0.03 \pm 0.01$ & $0.07 \pm 0.01$ & $155.17 \pm 61.10$ \\
        \end{tabular}}
        \end{table*}

\begin{table*}[] \centering
        \caption{Detailed results for \textbf{NLSY} Dataset. nan indicates that the model failed to optimize.}
        \label{atab:nlsy}
        \resizebox{.8\textwidth}{!}{
        \begin{tabular}{l|l|cc|ccc}
        \multicolumn{2}{c|}{} & \multicolumn{2}{c|}{\textbf{Performance Metrics}} & \multicolumn{3}{c}{\textbf{Fairness Metrics}} \\ 
         \multicolumn{2}{c|}{} & \textit{MSE (Avg Rank)} & \textit{SERA (Avg Rank)} & \textit{$\Delta BGL$ (Avg Rank)} & \textit{SP (Avg Rank)} & \textit{ID (Avg Rank)} \\ \hline
        \parbox[t]{2mm}{\multirow{4}{*}{\rotatebox[origin=c]{90}{Agnostic}}} & \textbf{XGB$_{MSE}$} & $1.27e+09 \pm 2.79e+08$ & $4.51e+11 \pm 1.17e+11$ & $6.42e+03 \pm 1.54e+03$ & $1.30e-01 \pm 9.61e-03$ & $8.87e+12 \pm 3.21e+12$ \\
         & \textbf{XGB$_{Huber}$} & $4.27e+09 \pm 5.25e+08$ & $1.35e+12 \pm 2.47e+11$ & $1.69e+04 \pm 2.51e+03$ & $0.00e+00 \pm 0.00e+00$ & $1.89e+13 \pm 3.42e+12$\\
         & \textbf{XGB$_{SERA}$} & $1.41e+09 \pm 3.02e+08$ & $4.45e+11 \pm 1.23e+11$ & $6.18e+03 \pm 1.22e+03$ & $1.35e-01 \pm 1.38e-02$  & $9.20e+12 \pm 3.35e+12$\\
         & \textbf{XGB$_{Indiv.}$} & $1.35e+09 \pm 2.80e+08$ & $4.82e+11 \pm 1.16e+11$ & $6.99e+03 \pm 1.32e+03$ & $1.76e-01 \pm 2.19e-02$  & $9.08e+12 \pm 2.54e+12$ \\ \hhline{|-|-|--|---|}
        \parbox[t]{2mm}{\multirow{7}{*}{\rotatebox[origin=c]{90}{Aware}}} & \textbf{Agarwal$_{MSE}$} & $nan$ & $nan$ & $nan$ & $nan$  & $nan$ \\ 
         & \textbf{Agarwal$_{SERA}$} & $nan$ & $nan$ & $nan$ & $nan$ & $nan$  \\
         & \textbf{Agarwal$_{ID}$} & $nan$ & $nan$ & $nan$ & $nan$ & $nan$ \\
         & \textbf{Calders$_{\alpha = 0}$} & $1.35e+09 \pm 2.52e+08$ & $4.60e+11 \pm 1.04e+11$ & $6.25e+03 \pm 1.55e+03$ & $1.30e-01 \pm 1.07e-02$ & $8.37e+12 \pm 2.86e+12$ \\
         & \textbf{Calders$_{\alpha = 5}$} & $1.35e+09 \pm 2.52e+08$ & $4.60e+11 \pm 1.04e+11$ & $6.24e+03 \pm 1.55e+03$ & $1.29e-01 \pm 1.08e-02$ & $8.37e+12 \pm 2.86e+12$ \\
        & \textbf{P'erez-Suay$_{1NN}$}~ & $6.45e+09 \pm 1.21e+10$ & $1.39e+12 \pm 1.48e+12$ & $9.19e+03 \pm 4.53e+03$ & $4.31e-02 \pm 3.13e-02$ & $1.41e+13 \pm 4.89e+12$ \\
        & \textbf{P'erez-Suay$_{XGB}$} & $2.05e+09 \pm 3.32e+08$ & $7.80e+11 \pm 1.57e+11$ & $5.63e+03 \pm 1.99e+03$ & $9.78e-02 \pm 2.15e-02$ & $1.16e+13 \pm 2.41e+12$ \\ \hhline{|-|-|--|---|}
        \parbox[t]{2mm}{\multirow{2}{*}{\rotatebox[origin=c]{90}{Ours}}} 
         & \textbf{IDBoost$_{0.5}$} & $1.55e+09 \pm 2.85e+08$ & $4.33e+11 \pm 1.18e+11$ & $6.98e+03 \pm 1.41e+03$ & $1.26e-01 \pm 1.36e-02$ & $7.78e+12 \pm 3.11e+12$ \\
         & \textbf{IDBoost$_{1.0}$} & $2.05e+09 \pm 3.23e+08$ & $4.79e+11 \pm 1.16e+11$ & $7.81e+03 \pm 1.96e+03$ & $1.19e-01 \pm 1.50e-02$ & $7.27e+12 \pm 2.98e+12$ \\
        \end{tabular}}
        \end{table*}

\begin{table*}[] \centering
        \caption{Detailed results for \textbf{COMPAS} Dataset}
        \label{atab:compas}
        \resizebox{.8\textwidth}{!}{
        \begin{tabular}{l|l|cc|ccc}
        \multicolumn{2}{c|}{} & \multicolumn{2}{c|}{\textbf{Performance Metrics}} & \multicolumn{3}{c}{\textbf{Fairness Metrics}} \\ 
         \multicolumn{2}{c|}{} & \textit{MSE (Avg Rank)} & \textit{SERA (Avg Rank)} & \textit{$\Delta BGL$ (Avg Rank)} & \textit{SP (Avg Rank)} & \textit{ID (Avg Rank)} \\ \hline
        \parbox[t]{2mm}{\multirow{4}{*}{\rotatebox[origin=c]{90}{Agnostic}}} & \textbf{XGB$_{MSE}$} & $0.14 \pm 0.01$ & $69.91 \pm 4.78$ & $0.03 \pm 0.01$ & $0.09 \pm 0.01$ & $176.56 \pm 66.78$ \\
         & \textbf{XGB$_{Huber}$} & $0.14 \pm 0.01$ & $70.18 \pm 5.08$ & $0.03 \pm 0.01$ & $0.09 \pm 0.01$ & $181.14 \pm 69.86$\\
         & \textbf{XGB$_{SERA}$} & $0.21 \pm 0.01$ & $50.54 \pm 4.01$ & $0.02 \pm 0.01$ & $0.09 \pm 0.00$  & $113.19 \pm 42.72$\\
         & \textbf{XGB$_{Indiv.}$} & $0.14 \pm 0.01$ & $73.56 \pm 5.08$ & $0.02 \pm 0.01$ & $0.10 \pm 0.01$  & $208.33 \pm 92.54$ \\ \hhline{|-|-|--|---|}
        \parbox[t]{2mm}{\multirow{7}{*}{\rotatebox[origin=c]{90}{Aware}}} & \textbf{Agarwal$_{MSE}$} & $0.16 \pm 0.01$ & $89.62 \pm 6.29$ & $0.02 \pm 0.01$ & $0.09 \pm 0.01$  & $192.53 \pm 77.53$ \\ 
         & \textbf{Agarwal$_{SERA}$} & $0.16 \pm 0.01$ & $89.00 \pm 5.89$ & $0.02 \pm 0.01$ & $0.09 \pm 0.01$ & $191.62 \pm 77.11$  \\
         & \textbf{Agarwal$_{ID}$} & $0.16 \pm 0.01$ & $89.25 \pm 5.98$ & $0.02 \pm 0.01$ & $0.09 \pm 0.01$ & $190.44 \pm 76.95$ \\
         & \textbf{Calders$_{\alpha = 0}$} & $0.35 \pm 0.01$ & $192.86 \pm 11.95$ & $0.05 \pm 0.01$ & $0.11 \pm 0.00$ & $321.25 \pm 75.08$ \\
         & \textbf{Calders$_{\alpha = 5}$} & $0.35 \pm 0.01$ & $192.93 \pm 11.95$ & $0.05 \pm 0.01$ & $0.11 \pm 0.00$ & $320.91 \pm 75.17$ \\
        & \textbf{P'erez-Suay$_{1NN}$}~ & $0.58 \pm 0.25$ & $353.76 \pm 210.23$ & $0.05 \pm 0.03$ & $0.08 \pm 0.01$ & $1086.43 \pm 887.43$ \\
        & \textbf{P'erez-Suay$_{XGB}$} & $0.35 \pm 0.15$ & $265.96 \pm 161.06$ & $0.04 \pm 0.03$ & $0.12 \pm 0.02$ & $724.09 \pm 437.94$ \\ \hhline{|-|-|--|---|}
        \parbox[t]{2mm}{\multirow{2}{*}{\rotatebox[origin=c]{90}{Ours}}} 
         & \textbf{IDBoost$_{0.5}$} & $0.26 \pm 0.01$ & $55.50 \pm 4.21$ & $0.02 \pm 0.01$ & $0.09 \pm 0.00$ & $96.37 \pm 34.46$ \\
         & \textbf{IDBoost$_{1.0}$} & $0.37 \pm 0.03$ & $73.19 \pm 6.21$ & $0.02 \pm 0.01$ & $0.09 \pm 0.01$ & $93.94 \pm 30.63$ \\
        \end{tabular}}
        \end{table*}
\end{document}